\documentclass{article} 
\usepackage{iclr2027_conference,times}

\usepackage{amsmath,amsfonts,bm}

\def\eqref#1{equation~\ref{#1}}

\def\1{\bm{1}}

\DeclareMathAlphabet{\mathsfit}{\encodingdefault}{\sfdefault}{m}{sl}
\SetMathAlphabet{\mathsfit}{bold}{\encodingdefault}{\sfdefault}{bx}{n}

\newcommand{\E}{\mathbb{E}}

\newcommand{\R}{\mathbb{R}}

\usepackage{hyperref}
\usepackage{url}

\usepackage[pdftex]{graphicx}  

\usepackage{float}
\usepackage{wrapfig}
\usepackage{algorithm}
\usepackage{algorithmic}

\usepackage{booktabs}  
\usepackage{tabularx}  

\usepackage{xcolor}  

\usepackage[font=normalsize,labelformat=simple]{subcaption}

\usepackage{listings}

\usepackage{amsmath,amssymb,amsthm}

\lstdefinestyle{pythoncode}{
    language=Python,
    basicstyle=\ttfamily\normalsize,
    columns=fullflexible,
    keepspaces=true,
    breaklines=true,
    showstringspaces=false,
    tabsize=4
}

\title{AdamPM: Memory-Efficient LLM Optimization via Moment Pooling}

\author{
Xiaowei Ji$^{1}$ \quad
Lishan Xu$^{1}$ \quad
Jianing An$^{1}$ \quad
Jie Luo$^{1}$ \quad
Wenjun Wu$^{1,2}$ \quad
Lei Huang$^{1,2}$\thanks{Corresponding author.} \\
$^{1}$SKLCCSE, School of Artificial Intelligence, Beihang University, Beijing, China \\
$^{2}$Hangzhou International Innovation Institute, Beihang University, Hangzhou, China \\
\texttt{\{jixiaowei,lishanxu,anjianing,wwj09315,luojie,huangleiai\}@buaa.edu.cn} \\
}

\newtheorem{theorem}{Theorem}[section]
\newtheorem{corollary}[theorem]{Corollary}
\newtheorem{lemma}[theorem]{Lemma}

\iclrfinalcopy 

\begin{document}

\maketitle

\begin{abstract}
Adam is widely used for training large language models (LLMs) owing to its strong optimization performance and training stability, but its element-wise estimation for the first and second moments incur substantial  memory overhead. 
Although numerous memory-efficient optimizers have been proposed, many incur additional computational overhead or require additional tuning of hyperparameters. 
Consequently, Adam remains a dominant and reliable choice for LLM pre-training. 
We propose AdamPM (Adam with Pooled Moments), a simple and memory-efficient variant of Adam that groups adjacent moment entries along the row or column dimension and represents each group by a single shared mean. 
AdamPM can well inherit Adam's hyperparameter settings and is readily integrated into existing Adam optimizer implementations, with less memory overhead. 
We systematically investigate pooling of both first and second moment estimates and identify second moment pooling as the effective design. 
With a group size of $s$, AdamPM reduces second moment memory to $1/s$ of the original requirement. 
Extensive experiments on language model pretraining, fine-tuning and multimodal learning demonstrate that AdamPM maintains model performance and training speed while substantially reducing optimizer-state memory. 
Our code is available at \url{https://github.com/jixiaowei95/AdamPM}.
\end{abstract}

\section{Introduction}

Adam~\citep{a:Adam} and its decoupled weight-decay variant AdamW~\citep{a:7}, which combine momentum with adaptive learning rates, have become the de facto standard for LLM training due to their strong convergence and stability. 
However, these advantages come at a substantial memory cost: 
Adam maintains first and second moment estimates for each parameter, requiring storage for two additional optimizer states per parameter.

As LLMs scale to trillions of parameters, this overhead becomes increasingly prohibitive.
Against this backdrop, memory-efficient optimizers have emerged as an important research direction, aiming to reduce training memory consumption through optimizer design while preserving the performance of conventional adaptive methods.
Representative approaches achieve this goal by  projecting gradients into low-rank subspaces~\citep{a:10,a:48}, factorizing the second moment state~\citep{a:38},  sharing the moment statistics~\citep{a:26,a:FOAM}, and  employing orthogonalization~\citep{a:56} or sign-based update rules~\citep{a:59} to eliminate the second moment state.

However, a recent survey from Stanford University~\citep{a:11} notes that, despite numerous optimizers claiming faster training or lower memory consumption, 
Adam remains the robust default for LLM pre-training, with its primary limitation being the memory overhead of optimizer states. This raises a natural question:

\textit{Before an optimizer consistently surpasses Adam in both performance and robustness, would a memory-efficient variant with only minimal modifications to Adam be a better practical choice?}

In this work, we systematically analyze the correlations and redundancy in Adam's optimizer states, and propose AdamPM, a minimally modified Adam variant that groups adjacent second moment entries and shares one statistic within each group.
Our main contributions are summarized as follows:

\begin{itemize}
\item We propose AdamPM, a memory-efficient optimizer for LLM training that reduces optimizer-state memory without sacrificing model performance or increasing training time. AdamPM makes only a minimal modification to vanilla Adam by pooling the second moment state. It can directly reuse Adam's hyperparameters, including the learning rate and decay rates, and is fully model-agnostic. These characteristics of AdamPM are  validated by experiments on LLM pre-training, language-model and vision-language model fine-tuning. 
\item We systematically investigate the pooling of both first and second moments through experiments.  
We find that the first moment contains substantially more information and is highly sensitive to pooling, explaining why most memory-efficient optimizers primarily compress the second moment. 
We introduce cosine similarity and relative error to quantify deviations from vanilla Adam, and show that AdamPM’s preconditioners and parameter updates become more similar to those of Adam as model scale grows, suggesting that AdamPM offers greater memory-efficiency benefits over Adam for larger models.
\end{itemize}

\section{Related Work}
\label{RelatedWork}

As LLMs scale to trillions of parameters, memory-efficient training has become a research priority. 
Existing memory-efficient optimizers reduce memory overhead primarily through matrix factorization, low-rank projection, constraint-based transformations, and parameter sharing.


\textbf{Matrix factorization.} These methods address reduction of memory by storing the factorized matrices, and the main differences involve to the optimization targets. 
For example, Adafactor~\citep{a:27} and SMMF~\citep{a:41} factorize optimizer states, the LoRA~\citep{a:25} family factorizes weight updates, and LORO~\citep{a:46} factorizes weight matrices. 
Although these methods improve memory efficiency, some are unsuitable for pre-training and limited to fine-tuning (e.g., LoRA), or 
have not yet been validated as extensively as Adam on LLMs with billions of parameters for pre-training (e.g., Adafactor and SMMF).
Our AdamPM, like Adam, works well for both pre-training and fine-tuning, and is validated by experiments on large-scale LLMs and vision-language models.

\textbf{Low-rank projection.} These methods aim at reducing memory by maintaining optimizer states in lower-dimensional spaces. 
They achieve this goal by random projections in APOLLO~\citep{a:10}, or by gradient projection based on either SVD in  GaLore~\citep{a:8} and Fira~\citep{a:47}, or online PCA~\citep{a:49}.  Although these methods reduce memory consumption,
they also incur additional computational overhead, such as periodic SVD computations in GaLore~\citep{a:8}, which increase training time. Besides, they often introduce new hyperparameters that are difficult to tune, such as the frequency of periodic subspace updates in Fira~\citep{a:47}. Our proposed AdamPM makes only minimal changes to Adam and preserves its desirable optimization properties,  while its parameter updates become increasingly similar to Adam's as model size grows, as demonstrated in 
Figure~\ref{figExperimentAnalysis} (d).

\textbf{Constraint-based transformations.} These methods aim to eliminate the need to store certain optimizer states by introducing constraint-based transformations, e.g., 
imposing orthogonality constraints on optimizer states in Muon~\citep{a:56} and SUMO~\citep{a:58}, normalizing gradients along the output feature dimension in SCALE~\citep{a:54}, and using sign-based momentum updates in Lion~\citep{a:59} and OLion~\citep{a:60}. 
These optimizers substantially modify Adam, preventing direct reuse of its hyperparameters and requiring additional tuning. AdamPM  makes only minimal changes to Adam by using moment pooling, and can directly inherit Adam’s hyperparameters, including the learning rate and
decay rates.

\textbf{Parameter sharing.}
Adam-mini~\citep{a:26} and FOAM~\citep{a:FOAM} also reduce optimizer-state memory by sharing moment estimates.
However, compared with Adam-mini~\citep{a:26}, AdamPM offers a more general grouping scheme, pooling adjacent second moment entries along rows or columns with a configurable group size $s$. 
Moreover, AdamPM is model-agnostic, whereas Adam-mini relies on model-specific information such as the hidden dimension and attention-head count.
AdamPM also differs from FOAM~\citep{a:FOAM}, which restricts sharing to power-of-two-sized groups of adjacent elements within each row, applies the same compression to first moments, and introduces residual correction. 
These additional operations complicate optimizer instantiation, and FOAM cannot adopt Adam’s hyperparameter settings.
Our AdamPM is fully model-agnostic, and can directly reuse Adam’s hyperparameters.

\section{Methodology}
\label{Methodology}

We first revisit the Adam optimizer in Section~\ref{VanillaAdam} and examine the local similarity of second moment statistics in Section~\ref{LocalSimilarity}. Motivated by these observations, we introduce AdamPM in Section~\ref{AdamPMOptimizer} and analyze its convergence in Section~\ref{sec:adampm-theory}.

\subsection{Vanilla Adam}
\label{VanillaAdam}


Given the stochastic gradients $G_t$ computed on a mini-batch, Adam maintains the first and second moment estimates, $M_t$ and $V_t$, using exponential moving averages (EMA).
In particular, it updates the weight $W_t$ as follows:
\begin{equation}
\begin{array}{l}
M_t=\beta_1M_{t-1}+(1-\beta_1)G_t, \\[3pt]
V_t=\beta_2V_{t-1}+(1-\beta_2)G_t\odot G_t, \\[3pt]
\hat{M_t} = \frac{M_t}{1-\beta_1^t}, \quad
\hat{V_t} = \frac{V_t}{1-\beta_2^t}, \\[3pt]
W_t = W_{t-1} - \eta_t \cdot
\frac{\widehat{M}_t}{\sqrt{\widehat{V}_t}+\epsilon}.
\end{array}
\label{eq:adam}
\end{equation}
where $\hat{M_t}$ and $\hat{V_t}$ are bias-corrections of the first and second moments respectively,
$\beta_1,\beta_2\in[0,1)$ are decay rates,
$\eta_t>0$ is the learning rate, and $\epsilon>0$ ensures numerical stability.
All operations are performed elementwise.

\begin{figure}[t]
\centering
\includegraphics[width=1.0\linewidth]{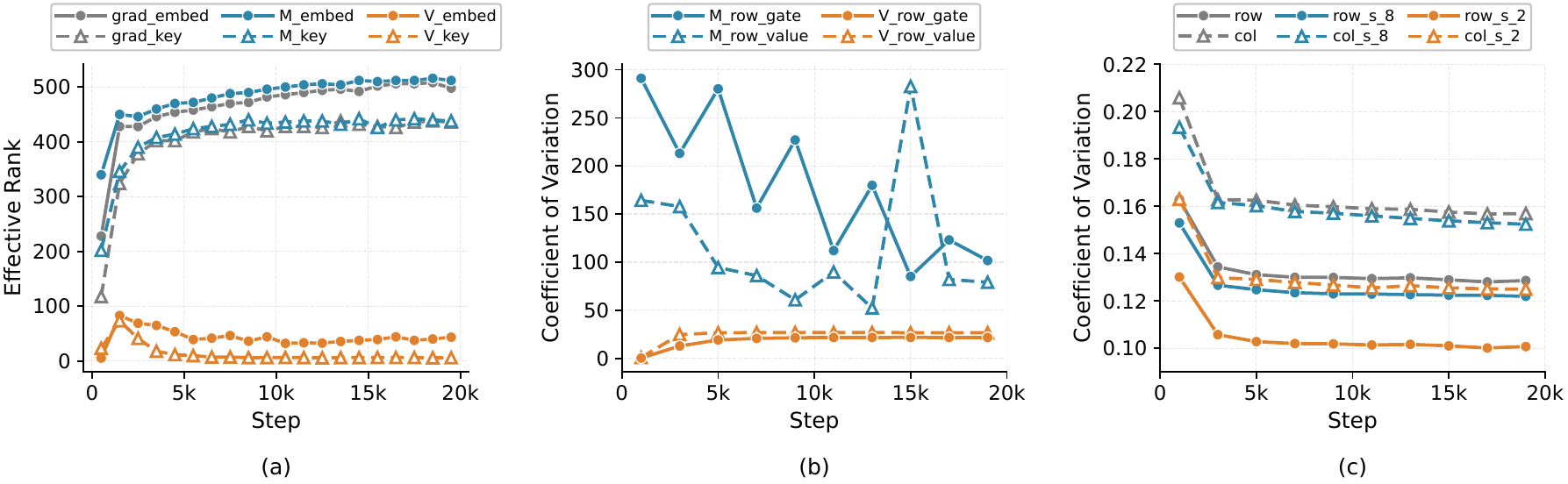}
\par
\setlength{\abovecaptionskip}{4pt}  
\caption{Evolution of effective ranks and coefficients of variation during LLaMA 130M pre-training. (a) Effective ranks of the gradients, first and second moments for \texttt{embed} and \texttt{key}. (b) Coefficients of variation of the first and second moments for \texttt{value} and \texttt{gate}. (c) Coefficients of variation of the preconditioner for \texttt{gate}. 
}
\label{figEffectiveRankAndCoefficientOfVariation}

\vspace{-8pt}

\end{figure}

\subsection{Local Similarity of Second Moment Statistics}
\label{LocalSimilarity}

Modern Transformer-based large language models widely adopt LayerNorm~\citep{a:LN} and RMSNorm~\citep{a:RMSNorm}. By reducing scale discrepancies across different weights and rescaling the loss landscape, normalization partially plays a role similar to adaptive learning-rate adjustment, which is achieved in Adam through second moment estimation. Consequently, the effective learning rates of different parameters tend to become more similar.

We examine the structural redundancy in Adam's optimizer states during the pretraining of LLaMA 130M model~\citep{a:73}. 
For illustration, Figure~\ref{figEffectiveRankAndCoefficientOfVariation} (a) reports the effective ranks of the gradient, first and second moments for \texttt{embed} and the fifth Transformer block's \texttt{key}. 
We define the effective rank of $A$ as $\operatorname{erank}(A)=\exp(-\sum_i p_i\log p_i)$, where $p_i=\sigma_i/\sum_j\sigma_j$ and $\sigma_i$ denotes its $i$-th singular value.  
Throughout training, the second moment consistently exhibits a substantially lower effective rank, indicating that its information is concentrated in only a few dominant directions. In contrast, the gradient and first moment contain richer directional information, with the first moment exhibiting an even higher effective rank than the gradient.
This observation is consistent with their roles in Adam: the first moment primarily determines the update direction, whereas the second moment controls the coordinate-wise update magnitude through the diagonal preconditioner. The second moment can tolerate moderate approximation as long as the resulting update scale remains sufficiently accurate. 
Adam-mini~\citep{a:26} also supports this view by sharing a single second moment statistic within each parameter block while retaining competitive performance. 


To assess local similarity during LLaMA 130M pretraining, we compute the coefficient of variation (CV) for the first moment, second moment, and diagonal preconditioner. 
For entries $\mathcal{X}$, CV is defined as $\mathrm{CV}(\mathcal{X})=\sigma_{\mathcal{X}}/|\mu_{\mathcal{X}}|$, where $\mu_{\mathcal{X}}$ and $\sigma_{\mathcal{X}}$ are their mean and standard deviation, respectively.

Figure~\ref{figEffectiveRankAndCoefficientOfVariation} (b) compares the CVs of the first and second moments for \texttt{value} and the MLP gate projection matrix \texttt{gate} in the fifth Transformer block, while Figure~\ref{figEffectiveRankAndCoefficientOfVariation} (c) reports the preconditioner CVs for \texttt{gate}. 
The labels \texttt{row} and \texttt{col} denote mean row-wise and column-wise CVs, respectively, while \texttt{s\_}$n$ denotes the mean CV across non-overlapping groups of $n$ adjacent elements along the corresponding dimension.
We observe that the first moment exhibits substantially higher CVs because it contains both positive and negative values and carries richer directional information, whereas the second moment has much lower CVs. 
Moreover, parameter updates depend directly on the diagonal preconditioner $\frac{1}{\sqrt{\widehat{V}_t}+\epsilon}$, whose inverse-square-root transformation further reduces this variation. The preconditioner exhibits consistently low CVs throughout training, indicating strong local similarity among neighboring elements, with even lower variation in smaller group sizes.

\subsection{AdamPM Optimizer}
\label{AdamPMOptimizer}

These observations motivate AdamPM, which reduces optimizer-state memory by applying \textbf{group-wise moment pooling} to Adam's second moment.
AdamPM pools every $s$ adjacent second moment entries within rows or columns, where $s$ denotes the group size.
By default, AdamPM performs pooling within each row, with rows corresponding to output features. In particular, 
given a weight matrix $W\in\mathbb{R}^{m\times n}$, each row is partitioned into $q=n/s$ groups if $n$ is divisible by $s$.
When $n$ is not divisible by $s$, the final group of squared gradients is padded with its mean.

\begin{algorithm}[htbp]
\caption{AdamPM Optimizer}
\label{alg:AdamPM}
\textbf{Input}: Initialized parameters $W_0 \in \mathbb{R}^{m \times n}$, total steps $T$, group size $s$, learning rate $\eta_t$, decay rates $\{\beta_1$, $\beta_2\}$, and numerical stabilization constant $\epsilon$.

\begin{algorithmic} 

\STATE Initialize $t \leftarrow 0$, $q \leftarrow n/s$
\hfill \{$q$ is the number of groups, assuming $s \mid n$\}

\STATE Initialize $M_0 \leftarrow \mathbf{0}_{m \times n}$, $\widetilde{V}_0 \leftarrow \mathbf{0}_{m \times q}$
\hfill \{First moment and compressed second moment\}

\setlength{\algorithmicindent}{2em}
\REPEAT
\setlength{\itemsep}{2pt}
    \STATE $t \leftarrow t+1$
    \STATE $G_t \leftarrow \nabla_W \mathcal{L}(W_{t-1})$
    \hfill \{Compute the gradient  $G_t$\}
    
    \STATE $M_t \leftarrow \beta_1 M_{t-1}+(1-\beta_1)G_t$
    \hfill \{Update the first moment $M_t$\}

    \STATE $S_t=G_t\odot G_t \in\mathbb{R}^{m\times n}$
    \hfill \{Compute the squared gradient\}

    \STATE $R_t \leftarrow S_t.\mathrm{reshape}\!\left(m,q,s\right) \in\mathbb{R}^{m\times q\times s}$
    \hfill \{Reshape the  squared gradient\}

\STATE {\setlength{\fboxsep}{1pt}
\colorbox{orange!35}{$\overline{S}_t \leftarrow \frac{1}{s}\sum_{k=1}^{s} \left(R_t\right)_{:,:,k} $}}
$\in \mathbb{R} ^ {m \times q}$
\hfill \{Mean within each group\}

    \STATE $\widetilde{V}_t \leftarrow \beta_2\widetilde{V}_{t-1} +(1-\beta_2)\overline{S}_t   \in \mathbb{R} ^ {m \times q}$
    \hfill \{Update the compressed second moment $\widetilde{V}_t$\}

    \STATE $\widehat{M}_t \leftarrow \frac{M_t}{1-\beta_1^t}$
    \hfill \{Correct bias in the first moment\}

    \STATE $\widehat{\widetilde{V}}_t \leftarrow \frac{\widetilde{V}_t}{1-\beta_2^t}$
    \hfill \{Correct bias in the compressed second moment\}

	\STATE {\setlength{\fboxsep}{1pt}
	\colorbox{orange!35}{$\widehat{V}_t \leftarrow \widehat{\widetilde{V}}_t \otimes \mathbf{1}_{1 \times s}   $}}  
    $\in\mathbb{R}^{m\times n}$
	\hfill \{Broadcast each group; $\otimes$ denotes the Kronecker product\}

    \STATE $W_t \leftarrow W_{t-1}   -\eta_t \cdot    \frac{\widehat{M}_t}    {\sqrt{\widehat{V}_t}+\epsilon}$
    \hfill \{Update weights\}

\UNTIL{$t=T$ or convergence}

\STATE \textbf{return} $W_t$

\end{algorithmic}
\end{algorithm}

The complete procedure is presented in Algorithm~\ref{alg:AdamPM}.
AdamPM performs moment pooling as follows. It first computes the element-wise squared gradient and reshapes it into $m\times q$ groups of $s$ adjacent elements. Then AdamPM computes the group means $\overline{S}_t$ and applies an EMA to obtain the compressed second moment $\widetilde{V}_t$. After bias correction, each group statistic is temporarily broadcast to its corresponding $s$ elements before the parameter update.

AdamPM pools second moments for all weight matrices, including \texttt{embed}, \texttt{lm\_head}, attention, and MLP weights. One-dimensional parameters, such as RMSNorm weights, are pooled as single-row or single-column matrices. Note that, unlike FOAM~\citep{a:FOAM}, AdamPM does not pool the first moment and follows Adam's element-wise update, since pooling the first moment naively leads to substantial performance degradation (Section~\ref{AnalysisofPreTrainingExperiments}).

The order of squared-gradient computation, pooling, and EMA accumulation is crucial for memory efficiency, 
as it is the EMA that incurs persistent optimizer-state storage overhead. 
By averaging squared gradients within each group before applying the EMA, 
AdamPM maintains only the compressed second-moment state $\widetilde{V}_t\in\mathbb{R}^{m\times q}$ throughout training, 
reducing second-moment memory usage to $1/s$ of that required by Adam. 
It therefore stores $mn$ first moment and $mn/s$ second moment elements, totaling $mn(1+1/s)$ optimizer-state elements, compared with $2mn$ for Adam.



\subsection{Convergence Analysis}
\label{sec:adampm-theory}

This section presents a finite-time stationarity analysis of AdamPM.
We assume that the objective is lower bounded and $L$-smooth. We further assume that stochastic gradients are conditionally unbiased with bounded noise variance and that the true gradient at each iterate is almost surely bounded in norm.

\begin{theorem}[Stationarity bound]
\label{thm:adampm-main}
Under the above conditions and with $\eta_t=\frac{\eta_0}{\sqrt t}$, we have

\begin{equation}
\begin{aligned}
\min_{1\le t\le T}\E\|\nabla f(W_t)\|_2^2
\le{}&K\frac{1+\log T}{\sqrt T}
+\left(3+\frac{4C}{\epsilon}\right)\sigma^2,
\end{aligned}
\label{eq:adampm-main-rate}
\end{equation}
where  $C$ and $\sigma^2$ are upper bounds on the norm of the true gradient and the noise variance respectively, $K$ is finite and independent of time step $T$ and row-wise group size $s$.
\end{theorem}

 The details of assumptions and proofs are provided in
Appendix~\ref{app:adampm-proofs}. 
This bound, including its noise residual, is independent of group size $s$. Consequently, AdamPM and Adam ($s=1$) share the same stationarity upper bound in our analysis.

\begin{corollary}[Vanishing stationarity bound]
\label{cor:adampm-relative}
Under Theorem~\ref{thm:adampm-main}, suppose additionally that
\begin{equation}
\begin{gathered}
\E[\|\xi_t\|_2^2\mid W_t]
\le\kappa\|\nabla f(W_t)\|_2^2, \quad 0\le\kappa<\frac{\epsilon}{16C+12\epsilon},
\end{gathered}
\label{eq:adampm-relative-condition}
\end{equation}
where $\xi_t=G_t-\nabla f(W_t)$.
Then
\begin{equation}
\begin{aligned}
&\min_{1\le t\le T}\E\|\nabla f(W_t)\|_2^2 \le \frac{K(1+\log T)}
{\left[1-4\kappa(3+4C/\epsilon)\right]\sqrt T}.
\end{aligned}
\label{eq:adampm-relative-rate}
\end{equation}
\end{corollary}

Thus, under the additional relative-noise condition, AdamPM achieves the same $O(\log T/{\sqrt{T})}$ convergence rate as Adam.
Proofs and explicit constants are provided in Appendix~\ref{app:adampm-proofs}.

\setlength{\intextsep}{12pt plus 2pt minus 2pt}

\begin{table}[!t]
\centering
\caption{Comparison of AdamPM, Adam, and representative memory-efficient optimizers for pre-training LLaMA models of various sizes on the C4 dataset. Validation perplexity is reported, with the estimated BF16 memory cost of model parameters and optimizer states. * denotes that the results are taken directly from the original papers, as they use the same pre-training setup, model scales, and token budgets.}
\label{tablePretraining}
\begin{tabular}{lcccc}
\toprule
Model Size & \textbf{60M} & \textbf{130M} & \textbf{350M} & \textbf{1.3B} \\
\midrule
Training Tokens & 1.1B & 2.2B & 6.4B & 13.1B\\
\midrule
Adam                  & 30.30 (0.325G)& 23.23 (0.749G)& 17.79 (2.05G) & 14.83 (7.48G)\\
\midrule
\textbf{AdamPM-2}   & 30.31 (0.270G)& 23.47 (0.624G)& 17.69 (1.71G) & 14.70 (6.24G)\\
\textbf{AdamPM-4}   & 30.25 (0.243G)& 22.99 (0.562G)& 17.64 (1.54G) & 14.66   (5.61G)\\
\textbf{AdamPM-8}   & 30.61 (0.230G)& 22.87 (0.531G)& 17.75 (1.46G) & 14.73 (5.30G)\\
\textbf{AdamPM-16}  & 30.57 (0.223G)& 23.17 (0.515G)& 17.86 (1.41G) & 14.93 (5.14G)\\
\textbf{AdamPM-inf} & 30.55 (0.217G)& 23.36 (0.500G)& 18.27 (1.37G) & 15.54 (4.99G)\\
\midrule
Adam-mini             & 30.82 (0.217G)& 23.57 (0.500G)& 18.59 (1.37G) & 15.72 (4.99G)\\
GaLore*                & 34.88 (0.24G) & 25.36 (0.52G) & 18.95 (1.22G) & 15.64 (4.38G) \\
APOLLO*                & 31.55 (0.24G) & 22.94 (0.52G) & 16.85 (1.22G) & 14.20 (4.38G) \\
SUMO*                  & 34.26 (0.23G) & 24.87 (0.51G) & 18.69 (1.16G) & 14.68 (3.84G) \\
LORO*                  & 33.96 (0.24G) & 24.59 (0.57G) & 18.84 (1.11G) & 15.19 (3.66G) \\
\bottomrule
\end{tabular}
\end{table}


\section{Experiments}
\label{Experiments}

We conduct comprehensive experiments to evaluate AdamPM across diverse training scenarios. Our evaluation includes LLaMA~\citep{a:73} pre-training, language-model fine-tuning on the GLUE benchmark~\citep{a:GLUE}, and multimodal fine-tuning using TinyLLaVA~\citep{zhou2024tinyllava}, a compact vision-language model. 
These experiments span different model architectures, parameter scales, and learning tasks, providing a systematic assessment of AdamPM in both pre-training and downstream adaptation. 
Throughout these experiments, AdamPM-$s$ denotes AdamPM with a group size of $s$.

\subsection{Pre-Training of LLaMA}

\paragraph{Experimental setup.}
We pre-train LLaMA models with 60M, 130M, 350M, and 1.3B parameters on the English subset of C4~\citep{a:74}. 
Results are summarized in Table~\ref{tablePretraining}. 
We evaluate performance using validation perplexity (PPL), where lower is better. 
We also report the combined memory footprint of model parameters and optimizer states, both stored in BF16.

For all model scales, we set $\beta_1=0.9$, $\beta_2=0.95$, and $\epsilon=10^{-8}$. 
For the 60M, 130M, 350M, and 1.3B models, both AdamPM and Adam use learning rates of $4\times10^{-3}$, $2\times10^{-3}$, $6\times10^{-4}$, and $3\times10^{-4}$, respectively. 
These values are selected for their robustness based on learning rate sweeps for Adam, such as those shown in Figure~\ref{figLearningRateRange}.
Detailed training configurations are provided in Appendix~\ref{app:LLaMA-Pre-Traininng}.

\paragraph{Baselines.}
We compare AdamPM with vanilla Adam, Adam-mini~\citep{a:26}, and several other representative memory-efficient optimizers, including GaLore~\citep{a:8}, which compresses optimizer states by projecting gradients into low-rank subspaces through SVD; APOLLO~\citep{a:10}, which uses random projection to construct subspaces; SUMO~\citep{a:58}, which orthogonalizes the projected first moment; and LORO~\citep{a:46}, which pre-trains low-rank-parameterized weight matrices through Riemannian optimization.
We evaluate Adam-mini in our own experiments, using the same learning rates, moment decay rates, and other hyperparameters as Adam and AdamPM.

We carefully tune the learning rate of Adam for each model scale. As shown in Table~\ref{tablePretraining}, the tuned Adam baseline achieves substantially stronger performance than the Adam baselines reported in several previous studies, such as SUMO~\citep{a:58} and LORO~\citep{a:46}. This observation highlights the importance of adequately tuning Adam when evaluating memory-efficient optimizers.

\paragraph{Main results.}
As shown in Table~\ref{tablePretraining}, AdamPM achieves perplexity comparable to that of carefully tuned vanilla Adam and frequently yields lower perplexity when the group size is not too large, while substantially reducing optimizer-state memory.
In particular, AdamPM-8 and AdamPM-16 reduce second moment memory by 87.5\% and 93.75\%, respectively, relative to Adam.
We further report the actual memory usage per GPU and compare the training time of AdamPM and Adam during LLaMA-1.3B pre-training in Appendix~\ref{app:LLaMA-Pre-Traininng}.
These results demonstrate that Adam's second moment statistics can be pooled to reduce memory usage without noticeably degrading pre-training performance.
We additionally evaluate an extreme pooling strategy, denoted as AdamPM-inf, in which all second moment entries within each row share a single scalar. As shown in Table~\ref{tablePretraining}, this aggressive pooling leads to modest performance degradation on the 350M and 1.3B models, while Adam-mini, which uses more aggressive second moment sharing, performs even worse on these two models.
Overall, group sizes between 2 and 8 generally yield robust performance.

\subsection{Analysis of Pre-Training Experiments}
\label{AnalysisofPreTrainingExperiments}

\begin{figure}[ht]
    \centering
    \includegraphics[width=0.8\linewidth]{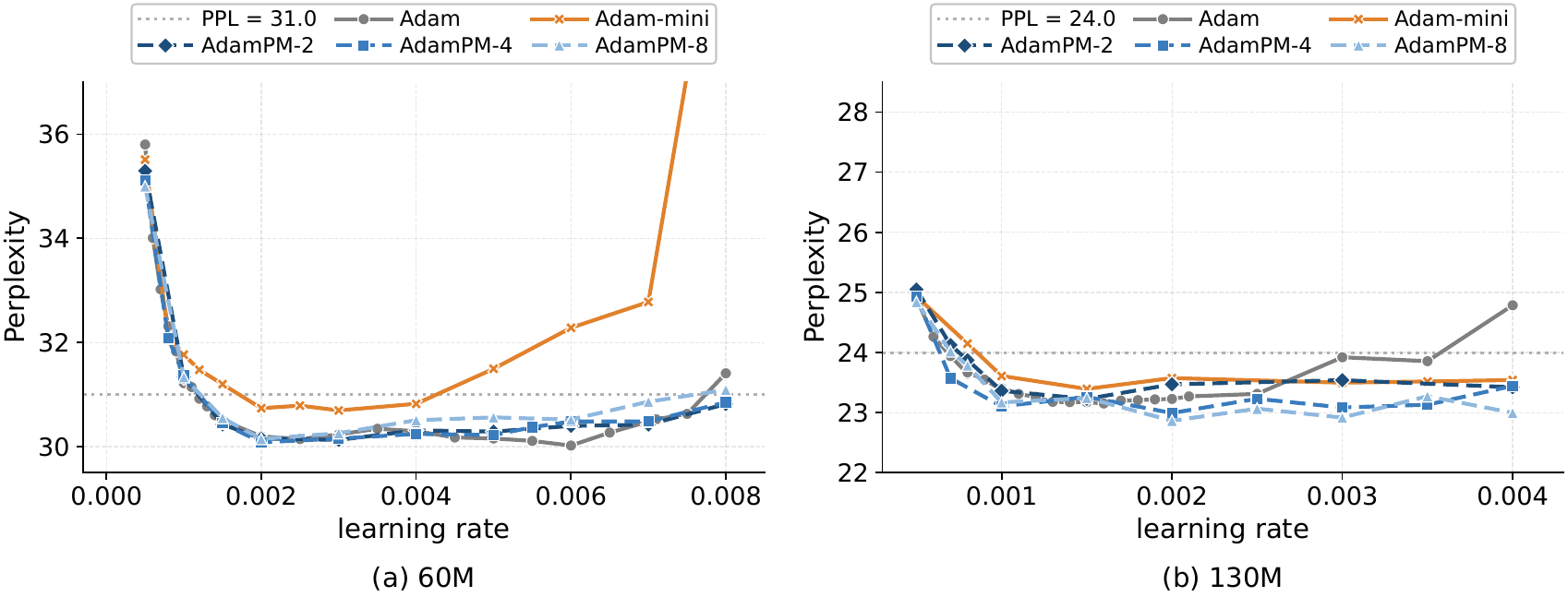}  
\par
\setlength{\abovecaptionskip}{1pt}  
    \caption{Learning rate sweeps for LLaMA 60M and 130M pre-training on C4}
    \label{figLearningRateRange}

\vspace{-8pt}

\end{figure}

\paragraph{Hyperparameter transferability.}
Figure~\ref{figLearningRateRange} presents learning rate sweeps for LLaMA 60M and 130M pre-training, with $\beta_1=0.9$ and $\beta_2=0.95$ fixed.
We evaluate AdamPM with group sizes of 2, 4, and 8, and observe that it maintains competitive perplexity over the learning rate ranges where Adam performs well.
These results demonstrate that AdamPM can inherit the well-tuned hyperparameters of Adam, making it easy to integrate into existing Adam optimizer implementations.
Furthermore, the LLaMA 60M results in Figure~\ref{figLearningRateRange} (a) show that AdamPM inherits Adam's learning rates more effectively than Adam-mini across all tested group sizes.


\setlength{\intextsep}{6pt} 

\begin{wraptable}{r}{0.45\textwidth}
\centering
\setlength{\tabcolsep}{2pt}
\setlength{\abovecaptionskip}{0pt}
\caption{Ablation study of pooling strategy.}
\label{tablePretraining2}

\begin{tabularx}{\linewidth}{@{}l*{3}{>{\centering\arraybackslash}X}@{}}
\toprule
Model Size & \textbf{60M} & \textbf{130M} & \textbf{350M} \\
\midrule
Adam                              & 30.30 & 23.23  & 17.79  \\
\midrule
\textbf{AdamPM-3}                 & 30.58 & 23.36  & 17.63  \\
\textbf{AdamPM-5}                 & 30.54 & 23.09  & 17.59  \\
\textbf{AdamPM-col-2}             & 30.45 & 23.40  & 17.56  \\
\textbf{AdamPM-col-3}             & 30.36 & 23.34  & 17.53  \\
\textbf{AdamPM-M-2}               & 59.99 & 206.35 & 49.26 \\
\textbf{AdamPM-M$^{\dagger}$-2}   & 42.35 & 30.12  & 23.71  \\
\bottomrule
\end{tabularx}
\end{wraptable}

\paragraph{Flexibility of the pooling strategy.}
Table~\ref{tablePretraining2} provides additional analyses of the group size and pooling direction. 
AdamPM-3 and AdamPM-5 achieve performance comparable to or even better than Adam. 
This result highlights AdamPM's flexibility: padding the final incomplete group with the mean of its existing entries allows any positive integer group size, rather than only powers or multiples of two. 
We also evaluate column-wise second moment pooling, denoted as AdamPM-col. 
This variant also achieves performance comparable to that of Adam, demonstrating that pooling is effective in either direction.

\paragraph{First- versus second-order moment pooling.}
We further evaluate the first moment pooling variant, denoted as AdamPM-M. 
Notably, directly pooling first moment entries leads to substantial performance degradation. 
As shown in Table~\ref{tablePretraining2}, AdamPM-M-2, which pools every two adjacent entries within each row of the first moment matrices, performs considerably worse than vanilla Adam.

Following the layer-wise strategy adopted by GaLore~\citep{a:8} and FRUGAL~\citep{a:61}, we revert the \texttt{embed}, \texttt{lm\_head} and \texttt{normalization} layers to vanilla Adam while applying first moment pooling to the remaining parameters; we denote this variant as AdamPM-M$^{\dagger}$.
The results in Table~\ref{tablePretraining2} show that although this modification improves performance over pooling the first moments of all weight parameters, a substantial performance gap relative to vanilla Adam remains.

Combined with the higher effective ranks and larger coefficients of variation observed for the first moment matrices in Section~\ref{LocalSimilarity}, 
these results indicate that the first moment carries substantial information regarding parameter update directions. 
Consequently, the first and second moments should not be considered equally compressible: 
the second moment exhibits considerable redundancy and can be effectively pooled, whereas directly pooling the first moment discards important optimization information and substantially degrades performance.

\paragraph{Training dynamics.}

Figure~\ref{figPerplexityComparision} presents the validation-perplexity curves of AdamPM-4, AdamPM-8, and Adam during the pre-training of the 1.3B LLaMA model. 

\setlength{\intextsep}{0pt} 

\begin{wrapfigure}{r}{0.4\textwidth}
    \centering
    \includegraphics[width=\linewidth]{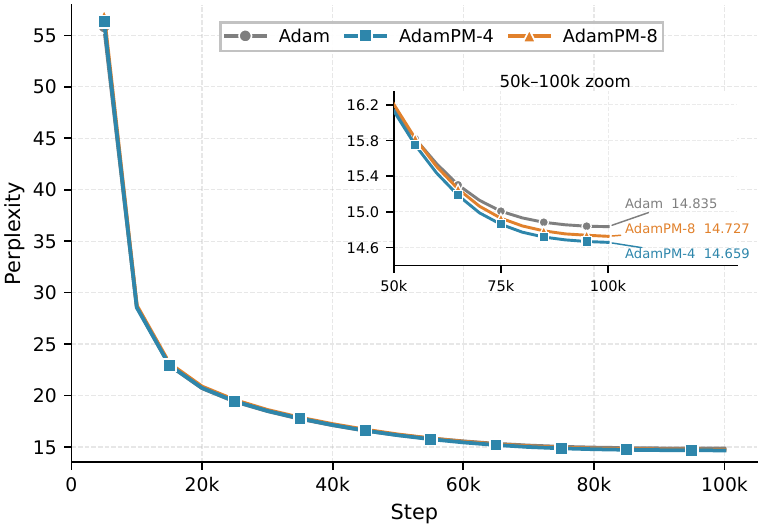}
    \setlength{\abovecaptionskip}{0pt}  
    \caption{Perplexity comparison between AdamPM and vanilla Adam during LLaMA 1.3B pre-training.}
    \label{figPerplexityComparision}
\end{wrapfigure}
The curves of the three optimizers nearly overlap throughout training,
indicating that moment pooling has little effect on the convergence trajectory of Adam.
AdamPM therefore preserves both the final perplexity and the overall optimization dynamics of vanilla Adam.

Figure~\ref{figExperimentAnalysis} reports the cosine similarity and relative error of the preconditioners and preconditioned updates before and after second moment pooling.
Figures~\ref{figExperimentAnalysis} (a) and (b) show that, throughout training, the cosine similarities remain consistently high, exceeding 0.993 for most steps, while the relative errors remain low. This indicates that AdamPM closely preserves both the directions and magnitudes of the updates produced by vanilla Adam, explaining their nearly identical convergence curves.

Figure~\ref{figExperimentAnalysis} (c) compares the cosine similarities of the preconditioned updates for group sizes $s\in\{2,4,8,16\}$. The similarity gradually decreases as $s$ increases, explaining the slight performance degradation observed with larger group sizes. Figure~\ref{figExperimentAnalysis} (d) shows that the cosine similarities increase with model size, suggesting that AdamPM is particularly suitable for larger models. Moreover, Figures~\ref{figExperimentAnalysis} (a), (b), and (d) consistently show that the preconditioned updates have higher cosine similarities and lower relative errors than the preconditioners themselves.

\setlength{\intextsep}{12pt plus 2pt minus 2pt}

\begin{figure}[ht]
    \centering
    \includegraphics[width=1.0\linewidth]{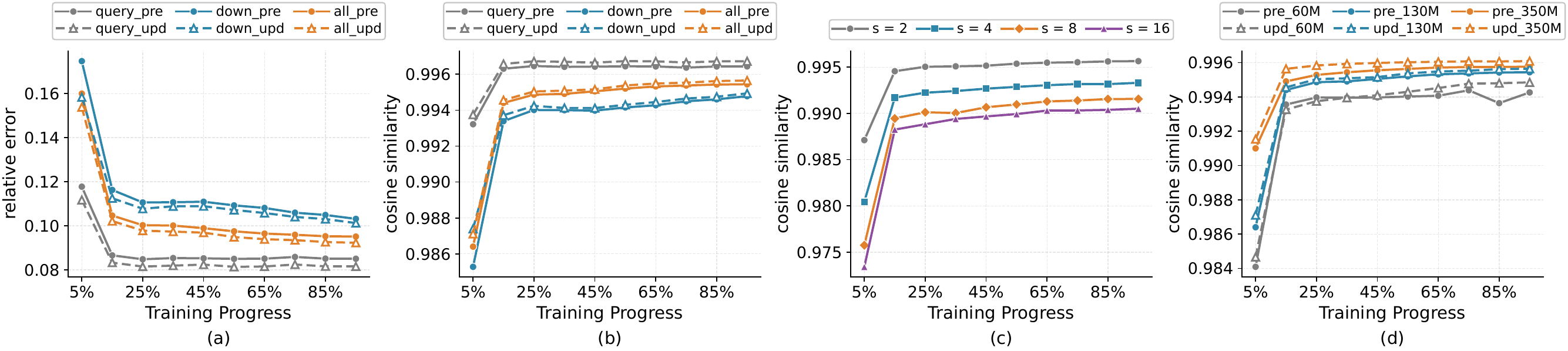}  
\par
\setlength{\abovecaptionskip}{4pt}  
    \caption{Comparison of the preconditioners $\frac{1}{\sqrt{\widehat{V}_t}+\epsilon}$ (denoted by \textit{pre}) and preconditioned updates $\frac{\widehat{M}_t}{\sqrt{\widehat{V}_t}+\epsilon}$ (denoted by \textit{upd}) before and after applying second moment pooling during LLaMA 130M pre-training. (a) Relative errors of the preconditioners and preconditioned updates for \texttt{query} and the MLP down projection matrix \texttt{down}, together with the averages across all weight matrices in the network. (b) Corresponding cosine similarities. (c) Average cosine similarity of preconditioned updates across all layers under different group sizes $s$. (d) Average cosine similarity across all layers for different model sizes.}
    \label{figExperimentAnalysis}

\vspace{-8pt}
    
\end{figure}

\subsection{Fine-Tuning on GLUE benchmark}

\begin{table}[ht]

\caption{Evaluating AdamPM for memory-efficient fine-tuning on GLUE benchmark using pretrained RoBERTa-base. We report the overall (matched and mismatched) accuracy for MNLI, Matthew's correlation coefficient for CoLA, Pearson correlation for STS-B, F1-score for MRPC, and accuracy for the remaining tasks.}
\label{tableFineTuning}
\centering
\setlength{\tabcolsep}{2pt}
\begin{tabular*}{\linewidth}{@{\extracolsep{\fill}}l|c|*{8}{c}|c@{}}
    \toprule
    & \textbf{Memory} & \textbf{CoLA} & \textbf{STS-B} & \textbf{MRPC} & \textbf{RTE} & \textbf{SST2} & \textbf{MNLI} & \textbf{QNLI} & \textbf{QQP} & \textbf{Avg} \\ 
    \midrule
    Adam                & 748M & 64.09 & 91.18 & 93.21 & 79.78 & 94.49 & 87.30 & 92.87 & \textbf{92.06} & 86.87 \\
    Adam-mini & 499M & 54.68 & 90.46 & 86.32 & 52.70 & 94.49 & 86.01 & 92.34 & 91.52 & 81.06 \\
    \midrule
    \textbf{AdamPM-2}   & 623M & \textbf{64.29} & 91.16 & 92.28 & 80.14 & 94.83 & 87.28 & 93.02 & 91.97 & 86.87 \\
    \textbf{AdamPM-4}   & 561M & 63.26 & 91.16 & \textbf{93.52} & 79.06 & 94.83 & 87.32 & \textbf{93.15} & 91.99 & 86.79 \\
    \textbf{AdamPM-8}   & 529M & 63.87 & \textbf{91.20} & 93.28 & \textbf{82.31} & 95.06 & \textbf{87.53} & 92.87 & 91.86 & \textbf{87.25} \\
    \textbf{AdamPM-16}   & 514M & 62.92 & \textbf{91.20} & 93.05 & 79.78 & \textbf{95.29} & 87.36 & 92.75 & 91.87 & 86.78 \\
    \bottomrule
\end{tabular*}
\end{table}

\paragraph{Experimental setup.}
We evaluate AdamPM by fine-tuning a pretrained RoBERTa-base~\citep{a:RoBERTa} model on the GLUE benchmark~\citep{a:GLUE}, using full fine-tuning with Adam and Adam-mini as baselines. 
For each task, we select Adam’s learning rate by grid search and reuse it for AdamPM and Adam-mini without further search. 
We vary the group size $s\in\{2,4,8,16\}$. 
All methods adopt the same hyperparameter strategy to ensure fair comparisons. 
Detailed experimental settings are provided in Appendix~\ref{app:hyperparameters of fine-tuning}.

\paragraph{Result analysis.}
As shown in Table~\ref{tableFineTuning}, AdamPM maintains average scores close to full fine-tuning with Adam across all evaluated group sizes, while memory usage decreases as group size $s$ increases. 
At $s=8$, its average score even exceeds the baseline, indicating that more extensive second-moment sharing can offer a favorable memory–performance trade-off. 
Although Adam-mini uses less memory, it performs worse when directly reusing Adam’s task-specific learning rates, particularly on smaller-data tasks such as CoLA and RTE. These results suggest that AdamPM can reuse Adam’s task-specific learning rates while maintaining competitive performance.

\subsection{Experiments on Large Multimodal Models}

\begin{table}[ht]
\vspace{-2pt}
\caption{Performance comparison of AdamPM with different group sizes on TinyLLaVA-1B.}
\label{tab:adampm_v_mllm}
\centering
\setlength{\tabcolsep}{2pt}
\begin{tabular*}{\linewidth}{@{\extracolsep{\fill}}l|c|*{8}{c}@{}}
\toprule
   & \textbf{Memory} & \textbf{VQAv2} & \textbf{MM-Vet} & \textbf{GQA} & \textbf{SQA} & \textbf{TextVQA} & \textbf{POPE} & \textbf{MME} & \textbf{MMMU} \\
\midrule
Adam  & 4.38G & 68.95 & 20.6 & 51.38 & 54.02 & \textbf{42.88} & 85.22 & 1330.60 & 29.80 \\
\midrule
\textbf{AdamPM-V-2} & 3.91G & 69.27 & \textbf{21.6} & 51.66 & 53.88 & 42.34 & 85.21 & 1329.66 & 30.30 \\
\textbf{AdamPM-V-4} & 3.68G & 69.55 & 19.5 & 52.06 & \textbf{54.14} & 42.41 & 85.58 & \textbf{1336.84} & 31.10 \\
\textbf{AdamPM-V-8} & 3.57G & 69.70 & 19.6 & 52.23 & 53.95 & 42.75 & 85.61 & 1320.59 & \textbf{31.30} \\
\textbf{AdamPM-V-16} & 3.51G & \textbf{69.84} & 19.2 & \textbf{52.61} & 54.07 & 42.41 & \textbf{85.82} & 1313.27 & 31.10 \\


\bottomrule
\end{tabular*}
\end{table}

\paragraph{Experimental setup.}
We evaluate AdamPM on TinyLLaVA-1B~\citep{zhou2024tinyllava}, comprising a SigLIP vision encoder, an MLP2x-GELU projector, and Qwen2-0.5B. We perform supervised fine-tuning on the vision-language instruction-tuning data released by LLaVA~\citep{liu2023visual}. During SFT, the language model and projector are fully fine-tuned, while the vision encoder remains frozen.
We compare Adam and AdamPM on VQAv2~\citep{goyal2017making}, MM-Vet~\citep{yu2023mm}, GQA~\citep{hudson2019gqa}, ScienceQA~\citep{lu2022learn}, TextVQA~\citep{singh2019towards}, POPE~\citep{li2023evaluating}, MME~\citep{fu2026mme}, and MMMU~\citep{yue2024mmmu} under identical training settings. Detailed training configurations are provided in Appendix~\ref{app:tinyllava_setup}.

\paragraph{Result analysis.}
As shown in Table~\ref{tab:adampm_v_mllm}, AdamPM substantially reduces static memory while maintaining comparable overall performance; AdamPM-16 saves 19.9\% memory over Adam and obtains the highest VQAv2, GQA, and POPE scores among the AdamPM variants.
Different group sizes perform best on different benchmarks, demonstrating that AdamPM provides a flexible trade-off between memory efficiency and multimodal performance.

\section{Conclusion}

We introduced AdamPM, a simple memory-efficient Adam variant that only pools second moment estimates. AdamPM reuses Adam's hyperparameters, remains model-agnostic, and incurs negligible training overhead. Our analysis shows that its preconditioners and parameter updates remain highly similar to Adam's, particularly at larger model scales. Experiments on LLM pre-training, fine-tuning, and multimodal models demonstrate substantial memory savings while preserving Adam-level performance and training efficiency.
These results establish second moment pooling used in AdamPM as a simple and practical approach to memory-efficient LLM training.

\paragraph{Limitations.}
Due to computational constraints, we have not conducted full training runs on models with 7B or more parameters. Our study considers pooling only along a single direction, either row-wise or column-wise. Pooling over two-dimensional groups and selective pooling that retains elementwise second moments in certain layers or modules, along with their effects on optimization, remain unexplored.

\clearpage

\subsection*{AI use statement}

In this work, we used generative AI tools solely for language polishing to improve the clarity and readability of the manuscript.
We have not used generative AI tools for any other research tasks.
We have reviewed all AI-assisted work.
We take responsibility for the final content of this work, including text, claims or artifacts produced with the aid of generative AI.

\subsection*{Reproducibility statement}

Our experiments use open-source model implementations and publicly available datasets for training and evaluation. To facilitate reproducibility, we provide the AdamPM implementation and experimental code at  \url{https://github.com/jixiaowei95/AdamPM}.

\bibliography{iclr2027_conference}
\bibliographystyle{iclr2027_conference}

\clearpage
\appendix

\renewcommand{\thefigure}{A\arabic{figure}}
\renewcommand{\thetable}{A\arabic{table}}
\setcounter{figure}{0}
\setcounter{table}{0}

\section{Lemmas and Proofs}
\label{app:adampm-proofs}

Our analysis draws on the momentum-error decomposition and descent-analysis framework used in FOAM \citep{a:FOAM}. We derive bounds tailored to AdamPM's grouped second moments, explicitly accounting for the random preconditioner and accumulated gradient noise.

\subsection{Setup and assumptions}
\label{app:adampm-setup}

We identify all parameters with $W\in\R^d$ and use the Euclidean norm;
matrix parameters are equivalently equipped with the Frobenius norm.
We state the assumptions used throughout the analysis.
\begin{enumerate}
\item \emph{L-Smoothness and lower boundedness.} The differentiable objective
$f:\R^d\to\R$ satisfies
\[
f(W)\ge f_*>-\infty,
\]
and
\[
\|\nabla f(W)-\nabla f(W')\|_2
\le L\|W-W'\|_2,
\quad W,W'\in\R^d.
\]
\item \emph{Conditional unbiasedness and bounded noise variance.}
For every $t$,
\[
G_t=\nabla f(W_t)+\xi_t,\quad
\E[\xi_t\mid W_t]=0,\quad
\E\|\xi_t\|_2^2\le\sigma^2.
\]
\item \emph{Bounded true gradients.} For every $t$,
$\|\nabla f(W_t)\|_2\le C$ almost surely.
\end{enumerate}

Let $P_s$ average and broadcast coordinates within each fixed row-wise
group of at most $s$ entries.
For any such group $I$,
\[
(P_s z)_i=\frac{1}{|I|}\sum_{j\in I}z_j,
\quad i\in I.
\]
A final partial group is averaged over its actual entries.
Padding that group with its own mean and discarding the padding gives
the same operator.
Ungrouped coordinates are represented by singleton groups.
Starting from deterministic $W_1$ and $M_0=V_0=0$, AdamPM performs
\begin{equation}
\begin{aligned}
M_t&=\beta_1M_{t-1}+(1-\beta_1)G_t,\\
V_t&=\beta_2V_{t-1}+(1-\beta_2)P_s(G_t^{\odot2}),\\
\widehat M_t&=\frac{M_t}{1-\beta_1^t},
\quad
\widehat V_t=\frac{V_t}{1-\beta_2^t},\\
W_{t+1}&=W_t-\eta_t
\frac{\widehat M_t}{\sqrt{\widehat V_t}+\epsilon}.
\end{aligned}
\label{eq:adampm-update}
\end{equation}
All vector operations are elementwise, and $\odot2$ denotes elementwise
squaring.
The parameters $\beta_1,\beta_2\in[0,1)$ and $\epsilon>0$ are fixed
independently of $T$.

Define
\begin{equation}
H_t=\operatorname{diag}\!\left(
\frac{1}{\sqrt{\widehat V_{t,i}}+\epsilon}
\right)_{i=1}^d.
\label{eq:adampm-H}
\end{equation}
Then $W_{t+1}=W_t-\eta_tH_t\widehat M_t$ and
$0\prec H_t\preceq\epsilon^{-1}I$.

\subsection{Moment and preconditioner bounds}

\begin{lemma}
\label{lem:adampm-moments}
For every $t\ge1$,
\begin{equation}
\begin{aligned}
\E\|G_t\|_2^2&\le C^2+\sigma^2, \quad
\E\|\widehat M_t\|_2^2\le C^2+\sigma^2,\quad
\E\|H_t\widehat M_t\|_2^2
\le\frac{C^2+\sigma^2}{\epsilon^2}.
\end{aligned}
\label{eq:adampm-moment-bounds}
\end{equation}
Define the accumulated noise energy
\begin{equation}
Z_t=\frac{1-\beta_2}{1-\beta_2^t}
\sum_{j=1}^t\beta_2^{t-j}\|\xi_j\|_2^2.
\label{eq:adampm-Z}
\end{equation}
Then $\E Z_t\le\sigma^2$, and, almost surely,
\begin{equation}
\nabla f(W_t)^\top H_t\nabla f(W_t)
\ge
\frac{\|\nabla f(W_t)\|_2^2-Z_t}{2C+\epsilon}.
\label{eq:adampm-energy-lower}
\end{equation}
\end{lemma}

\begin{proof}
Conditional unbiasedness gives
$\E\langle\nabla f(W_t),\xi_t\rangle=0$.
Consequently,
\[
\E\|G_t\|_2^2
=\E\|\nabla f(W_t)\|_2^2+\E\|\xi_t\|_2^2
\le C^2+\sigma^2.
\]
The bias-corrected first moment satisfies
\begin{equation}
\widehat M_t=\sum_{j=1}^t w_{t,j}G_j,
\quad
w_{t,j}=\frac{(1-\beta_1)\beta_1^{t-j}}{1-\beta_1^t}.
\label{eq:adampm-weights}
\end{equation}
Jensen's inequality and $H_t\preceq\epsilon^{-1}I$ prove
Eq.~\ref{eq:adampm-moment-bounds}.
The same convex-combination argument gives $\E Z_t\le\sigma^2$.

The inequality $(a+b)^2\le2a^2+2b^2$ and group averaging yield
\[
\begin{aligned}
\bigl(P_s(G_j^{\odot2})\bigr)_i
&\le
2\bigl(P_s(\nabla f(W_j)^{\odot2})\bigr)_i
+2\bigl(P_s(\xi_j^{\odot2})\bigr)_i\\
&\le2C^2+2\|\xi_j\|_2^2.
\end{aligned}
\]
Here, each group average is bounded by the squared norm of the full vector.
Applying the bias-corrected second-moment weights gives
\begin{equation}
\widehat V_{t,i}\le2C^2+2Z_t.
\label{eq:adampm-v-upper}
\end{equation}
On $\{Z_t\le C^2\}$, every denominator is at most $2C+\epsilon$,
which implies Eq.~\ref{eq:adampm-energy-lower}.
On $\{Z_t>C^2\}$, its right-hand side is nonpositive, whereas its
left-hand side is nonnegative.
\end{proof}

\subsection{Momentum error}

Decompose the first-moment error as
\begin{equation}
\begin{aligned}
e_t&=\widehat M_t-\nabla f(W_t)=a_t+n_t,\\
a_t&=\sum_{j=1}^t w_{t,j}
\bigl(\nabla f(W_j)-\nabla f(W_t)\bigr),\\
n_t&=\sum_{j=1}^t w_{t,j}\xi_j.
\end{aligned}
\label{eq:adampm-error}
\end{equation}

\begin{lemma}
\label{lem:adampm-drift}
For $\eta_t=\frac{\eta_0}{\sqrt t}$, define
\begin{equation}
C_{\mathrm m}
=\frac{L\eta_0\sqrt{C^2+\sigma^2}}{\epsilon}
\frac{\beta_1\sqrt{2-\beta_1}}{(1-\beta_1)^{3/2}}.
\label{eq:adampm-Cm}
\end{equation}
Then, for every $t\ge1$,
\begin{equation}
\begin{aligned}
\E\|a_t\|_2^2&\le\frac{C_{\mathrm m}^2}{t},\\
\E\|n_t\|_2^2
&\le\sum_{j=1}^t w_{t,j}\E\|\xi_j\|_2^2
\le\sigma^2.
\end{aligned}
\label{eq:adampm-error-components}
\end{equation}
In particular,
\begin{equation}
\E\|e_t\|_2^2\le\frac{2C_{\mathrm m}^2}{t}+2\sigma^2.
\label{eq:adampm-error-bound}
\end{equation}
\end{lemma}

\begin{proof}
Expanding the gradient differences and exchanging finite sums gives
\[
a_t=\sum_{r=1}^{t-1}
\left(\sum_{j=1}^r w_{t,j}\right)
\bigl(\nabla f(W_r)-\nabla f(W_{r+1})\bigr).
\]
Moreover, for $r<t$,
\[
\sum_{j=1}^r w_{t,j}
=\beta_1^{t-r}\frac{1-\beta_1^r}{1-\beta_1^t}
\le\beta_1^{t-r}.
\]
Smoothness, the $L^2$ triangle inequality, and
Eq.~\ref{eq:adampm-moment-bounds} therefore imply
\begin{equation}
\begin{aligned}
\bigl(\E\|a_t\|_2^2\bigr)^{1/2}
&\le L\sum_{r=1}^{t-1}\eta_r
\left(\sum_{j=1}^r w_{t,j}\right)
\bigl(\E\|H_r\widehat M_r\|_2^2\bigr)^{1/2}\\
&\le\frac{L\sqrt{C^2+\sigma^2}}{\epsilon}
\sum_{r=1}^{t-1}\eta_r\beta_1^{t-r}.
\end{aligned}
\label{eq:adampm-drift-general}
\end{equation}
For $1\le k<t$, the inequality $t/(t-k)\le k+1$ gives
\[
\sum_{r=1}^{t-1}\frac{\beta_1^{t-r}}{\sqrt r}
\le\frac{1}{\sqrt t}
\sum_{k=1}^{\infty}\beta_1^k\sqrt{k+1}.
\]
By Cauchy--Schwarz and geometric-series identities,
\[
\begin{aligned}
\sum_{k=1}^{\infty}\beta_1^k\sqrt{k+1}
&\le
\left(\sum_{k=1}^{\infty}\beta_1^k\right)^{1/2}
\left(\sum_{k=1}^{\infty}(k+1)\beta_1^k\right)^{1/2}\\
&=\frac{\beta_1\sqrt{2-\beta_1}}{(1-\beta_1)^{3/2}}.
\end{aligned}
\]
This proves the drift bound.
Jensen's inequality gives
\[
\E\|n_t\|_2^2
\le\sum_{j=1}^t w_{t,j}\E\|\xi_j\|_2^2.
\]
Finally, use $\|a_t+n_t\|_2^2\le2\|a_t\|_2^2+2\|n_t\|_2^2$.
\end{proof}

\subsection{Proof of Theorem~\ref{thm:adampm-main}}
\label{app:adampm-main-proof}

\begin{proof}
Young's inequality in the inner product induced by $H_t$ yields
\begin{align}
\langle\nabla f(W_t),H_t\widehat M_t\rangle
&\ge\frac12\nabla f(W_t)^\top H_t\nabla f(W_t)
-\frac12 e_t^\top H_te_t\notag\\
&\ge\frac{\|\nabla f(W_t)\|_2^2-Z_t}{2(2C+\epsilon)}
-\frac{\|e_t\|_2^2}{2\epsilon}.
\label{eq:adampm-direction}
\end{align}
Applying smoothness to Eq.~\ref{eq:adampm-update} and taking expectations
therefore gives
\begin{equation}
\begin{aligned}
\frac{\eta_t}{2(2C+\epsilon)}\E\|\nabla f(W_t)\|_2^2
\le{}&\E[f(W_t)-f(W_{t+1})]
+\frac{\eta_t\E Z_t}{2(2C+\epsilon)}\\
&+\frac{\eta_t}{2\epsilon}\E\|e_t\|_2^2
+\frac{L(C^2+\sigma^2)}{2\epsilon^2}\eta_t^2.
\end{aligned}
\label{eq:adampm-one-step}
\end{equation}
Summing and using lower boundedness gives the common starting point for
both results:
\begin{equation}
\begin{aligned}
\sum_{t=1}^T\eta_t\E\|\nabla f(W_t)\|_2^2
\le{}&2(2C+\epsilon)\bigl(f(W_1)-f_*\bigr)
+\sum_{t=1}^T\eta_t\E Z_t\\
&+\frac{2C+\epsilon}{\epsilon}
\sum_{t=1}^T\eta_t\E\|e_t\|_2^2\\
&+\frac{L(2C+\epsilon)(C^2+\sigma^2)}{\epsilon^2}
\sum_{t=1}^T\eta_t^2.
\end{aligned}
\label{eq:adampm-master}
\end{equation}
Substituting $\E Z_t\le\sigma^2$ and
Eq.~\ref{eq:adampm-error-bound}, we obtain
\begin{equation}
\begin{aligned}
\sum_{t=1}^T\eta_t\E\|\nabla f(W_t)\|_2^2
\le{}&2(2C+\epsilon)\bigl(f(W_1)-f_*\bigr)+\left(3+\frac{4C}{\epsilon}\right)\sigma^2
\sum_{t=1}^T\eta_t\\
&\quad+\frac{2(2C+\epsilon)C_{\mathrm m}^2}{\epsilon}
\sum_{t=1}^T\frac{\eta_t}{t}\\
&\quad+\frac{L(2C+\epsilon)(C^2+\sigma^2)}{\epsilon^2}
\sum_{t=1}^T\eta_t^2.
\end{aligned}
\label{eq:adampm-bounded-sum}
\end{equation}
For $\eta_t=\frac{\eta_0}{\sqrt t}$, elementary sum bounds give
\begin{equation}
\begin{gathered}
\sum_{t=1}^T\eta_t\ge\eta_0\sqrt T,\quad
\sum_{t=1}^T\frac{\eta_t}{t}\le3\eta_0,\\
\sum_{t=1}^T\eta_t^2\le\eta_0^2(1+\log T).
\end{gathered}
\label{eq:adampm-sums}
\end{equation}
Indeed, $t^{-1/2}\ge T^{-1/2}$ for $t\le T$,
$\sum_{t\ge1}t^{-3/2}\le1+\int_1^\infty x^{-3/2}\,dx=3$,
and $\sum_{t=1}^Tt^{-1}\le1+\log T$.
Divide Eq.~\ref{eq:adampm-bounded-sum} by $\sum_{t=1}^T\eta_t$ and use
\[
\min_{1\le t\le T}\E\|\nabla f(W_t)\|_2^2
\le
\frac{\sum_{t=1}^T\eta_t\E\|\nabla f(W_t)\|_2^2}
{\sum_{t=1}^T\eta_t}.
\]
Equation~\ref{eq:adampm-main-rate} follows with
\begin{equation}
\begin{aligned}
K={}&
\frac{2(2C+\epsilon)\bigl(f(W_1)-f_*\bigr)}{\eta_0}
+\frac{6(2C+\epsilon)C_{\mathrm m}^2}{\epsilon}\\
&+\frac{L\eta_0(2C+\epsilon)(C^2+\sigma^2)}{\epsilon^2}.
\end{aligned}
\label{eq:adampm-K}
\end{equation}
For common assumption constants and initial objective gap, this expression
is independent of $T$, $\beta_2$, and the fixed grouping.
\end{proof}

\subsection{A weighted-history bound}

\begin{lemma}
\label{lem:adampm-history}
For $\beta\in[0,1)$, set
\[
w_{t,j}^{(\beta)}
=\frac{(1-\beta)\beta^{t-j}}{1-\beta^t}.
\]
For $\eta_t=\frac{\eta_0}{\sqrt t}$ and any nonnegative sequence $(u_j)$,
\begin{equation}
\sum_{t=1}^T\eta_t\sum_{j=1}^t w_{t,j}^{(\beta)}u_j
\le4\sum_{j=1}^T\eta_j u_j.
\label{eq:adampm-history}
\end{equation}
The constant is uniform over $\beta\in[0,1)$.
\end{lemma}

\begin{proof}
For $0<\beta<1$,
\[
w_{t,j}^{(\beta)}
=(1-\beta)\beta^{t-j}
+\beta^{t-j}\frac{(1-\beta)\beta^t}{1-\beta^t}.
\]
Since
\[
\frac{(1-\beta)\beta^t}{1-\beta^t}
=\frac{\beta^t}{\sum_{r=0}^{t-1}\beta^r}
\le\frac{\beta}{t}\le\frac1t,
\]
we have $w_{t,j}^{(\beta)}\le(1-\beta)\beta^{t-j}+1/t$.
Using the monotonicity of $\eta_t$ and an integral bound,
\begin{align*}
\sum_{t=j}^T\eta_t w_{t,j}^{(\beta)}
&\le\eta_j(1-\beta)\sum_{r=0}^{\infty}\beta^r
+\eta_0\sum_{t=j}^{\infty}t^{-3/2}\\
&\le\eta_j+\eta_0\bigl(j^{-3/2}+2j^{-1/2}\bigr)
\le4\eta_j.
\end{align*}
Exchanging the finite sums proves Eq.~\ref{eq:adampm-history}.
\end{proof}

\subsection{Proof of Corollary~\ref{cor:adampm-relative}}

\begin{proof}
Let
\[
S_T=\sum_{t=1}^T\eta_t\E\|\nabla f(W_t)\|_2^2.
\]
Taking expectations in Eq.~\ref{eq:adampm-relative-condition} gives
$\E\|\xi_t\|_2^2\le\kappa\E\|\nabla f(W_t)\|_2^2$.
Apply Lemma~\ref{lem:adampm-history} to the second-moment weights and then
to the first-moment weights in Eq.~\ref{eq:adampm-error-components}:
\begin{equation}
\sum_{t=1}^T\eta_t\E Z_t\le4\kappa S_T,
\quad
\sum_{t=1}^T\eta_t\E\|n_t\|_2^2\le4\kappa S_T.
\label{eq:adampm-relative-histories}
\end{equation}
Combining these bounds with the drift estimate yields
\begin{equation}
\begin{aligned}
\sum_{t=1}^T\eta_t\E\|e_t\|_2^2
&\le2C_{\mathrm m}^2\sum_{t=1}^T\frac{\eta_t}{t}+8\kappa S_T\\
&\le6\eta_0C_{\mathrm m}^2+8\kappa S_T.
\end{aligned}
\label{eq:adampm-relative-error}
\end{equation}
Substitution into Eq.~\ref{eq:adampm-master}, before replacing
$\E Z_t$ by a uniform variance bound, gives
\begin{equation}
\begin{aligned}
&\left[1-4\kappa\left(3+\frac{4C}{\epsilon}\right)\right]S_T\\
&\quad\le2(2C+\epsilon)\bigl(f(W_1)-f_*\bigr)
+\frac{6\eta_0(2C+\epsilon)C_{\mathrm m}^2}{\epsilon}\\
&\qquad+\frac{L(2C+\epsilon)(C^2+\sigma^2)}{\epsilon^2}
\sum_{t=1}^T\eta_t^2.
\end{aligned}
\label{eq:adampm-relative-absorption}
\end{equation}
The coefficient on the left is positive by
Eq.~\ref{eq:adampm-relative-condition}.
The right-hand side is at most $\eta_0K(1+\log T)$ by
Eqs.~\ref{eq:adampm-sums} and~\ref{eq:adampm-K}.
Dividing by this coefficient and by $\sum_{t=1}^T\eta_t$ proves
Eq.~\ref{eq:adampm-relative-rate}.
\end{proof}

\section{Optimizer Instantiation}
\label{app:optimizer_instantiation}

The following code illustrates how to instantiate the optimizers used in our experiments, highlighting AdamPM's ease of use.

\begin{lstlisting}[
language=Python,
backgroundcolor=\color{gray!10},
basicstyle=\ttfamily\normalsize,
frame=single,
rulecolor=\color{gray!30},
linewidth=\linewidth,
framesep=3pt,
framerule=0.4pt,
xleftmargin=3.4pt,
xrightmargin=3.4pt,
framexleftmargin=0pt,
framexrightmargin=0pt,
breaklines=true,
showstringspaces=false,
escapeinside={(*@}{@*)}
]
# Adam
optimizer = torch.optim.Adam(
  (*@\textcolor{blue!100}{\strut\texttt{trainable\_params}}@*),
  lr=args.lr,
  betas=(args.beta1, args.beta2),
)


# AdamPM
optimizer = AdamPM(
  (*@\textcolor{blue!100}{\strut\texttt{trainable\_params}}@*),
  lr=args.lr,
  betas=(args.beta1, args.beta2),
)


# Adam-mini
optimizer = Adam_mini(
  (*@\textcolor{orange!100}{\strut\texttt{model.named\_parameters()}}@*),
  lr=args.lr,
  betas=(args.beta1, args.beta2),
  (*@\textcolor{red!100}{\strut\texttt{dim=model\_config.hidden\_size}}@*),
  (*@\textcolor{red!100}{\strut\texttt{n\_heads=model\_config.num\_attention\_heads}}@*),
)


# FOAM
optimizer = FOAM(
  (*@\textcolor{green!100}{\strut\texttt{param\_groups}}@*),
  lr=args.lr,
  betas=(args.beta1, args.beta2),
  (*@\textcolor{red!100}{\strut\texttt{res\_scale=args.res\_scale}}@*),
  (*@\textcolor{red!100}{\strut\texttt{no\_norm\_limit=args.no\_norm\_limit}}@*),
)


# GaLore
optimizer = GaLoreAdamW(
  (*@\textcolor{green!100}{\strut\texttt{param\_groups}}@*),
  lr=args.lr,
  betas=(args.beta1, args.beta2),
)
\end{lstlisting}

\section{Details of the experiments}

\setlength{\intextsep}{12pt plus 2pt minus 2pt}

\subsection{Details of LLaMA pre-training}
\label{app:LLaMA-Pre-Traininng}

\paragraph{Hyperparameters.}

Table~\ref{table:LLaMAHyperparameters} summarizes the architectural and optimization hyperparameters for pre-training LLaMA models of different sizes.
For each model size, Adam, Adam-mini, and AdamPM use identical hyperparameter settings, regardless of the AdamPM variant or group size.
The learning rates for the 60M, 130M, 350M, and 1.3B models are $4\times10^{-3}$, $2\times10^{-3}$, $6\times10^{-4}$, and $3\times10^{-4}$, respectively.
These values are selected for their robustness based on learning-rate sweeps for Adam, such as those shown in Figure~\ref{figLearningRateRange}.
Across all model sizes, we set $\beta_1=0.9$, $\beta_2=0.95$, and $\epsilon=10^{-8}$.

We employ a cosine learning rate decay schedule~\citep{a:75} with learning rate warmup~\citep{a:76}. The first 10\% of the total training steps are used for warmup, followed by cosine decay for the remaining steps. Weight decay is disabled in all experiments.

\begin{table}[h]
\caption{The hyperparameters for LLaMA pre-training.}
\label{table:LLaMAHyperparameters}
\centering
\begin{tabularx}{\linewidth}{
    @{}
    c
    >{\centering\arraybackslash}X
    >{\centering\arraybackslash}X
    >{\centering\arraybackslash}X
    >{\centering\arraybackslash}X
    >{\centering\arraybackslash}X
    @{}
}
\toprule
Model Size              & \textbf{60M} & \textbf{130M} & \textbf{350M} & \textbf{1.3B} \\
\midrule
Hidden dimension        & 512               & 768                & 1024               & 2048               \\
Intermediate dimension  & 1376              & 2048               & 2736               & 5461               \\
Attention Heads         & 8                 & 12                 & 16                 & 32                 \\  
Decoder layers          & 8                 & 12                 & 24                 & 24                 \\
Learning Rate           & $4\times10^{-3}$  &  $2\times10^{-3}$  &  $6\times10^{-4}$  &  $3\times10^{-4}$  \\
$\beta_1$               & \multicolumn{4}{c}{$0.9$}     \\
$\beta_2$               & \multicolumn{4}{c}{$0.95$}    \\
$\epsilon$              & \multicolumn{4}{c}{$10^{-8}$} \\
Weight Decay            & \multicolumn{4}{c}{$0$} \\
Learning Rate Scheduler & \multicolumn{4}{c}{$Cosine$} \\
Warmup                  & \multicolumn{4}{c}{$10\%$} \\
Max seq. len.          & \multicolumn{4}{c}{$256$} \\
\bottomrule
\end{tabularx}
\end{table}

\paragraph{Training Efficiency.}
The implementation of AdamPM performs group-wise second moment pooling entirely through matrix-level tensor operations, including reshaping, averaging, and broadcasting, enabling efficient vectorized computation without element-wise Python \texttt{for} loops.
Its computational cost is therefore nearly identical to that of vanilla Adam.
In the pre-training of LLaMA 1.3B with a batch size of 32 and 4$\times$ NVIDIA A100 40GB GPUs, vanilla Adam and AdamPM require 116.2 and 116.4 hours, respectively, confirming that AdamPM introduces negligible training-time overhead.
Moreover, the reduced memory footprint may allow a larger batch size under the same hardware constraints, improving hardware utilization and further accelerating training.

\paragraph{Actual training memory usage.}
We measure the actual training memory consumption per GPU for Adam and AdamPM during LLaMA-1.3B pre-training on NVIDIA A100 40GB GPUs with a batch size of 32. The results are reported in Table~\ref{table:TrueMemory}.

\begin{table}[htbp]
\centering
\caption{Measured memory usage per GPU during LLaMA-1.3B pre-training.}
\label{table:TrueMemory}
\begin{tabular}{lcccc}
\toprule
Optimizer & Memory Cost \\
\midrule
Adam                &   38.11 G\\
\textbf{AdamPM-2}   &   36.73 G\\
\textbf{AdamPM-4}   &   35.96 G\\
\textbf{AdamPM-8}   &   35.43 G\\
\textbf{AdamPM-16}  &   35.19 G\\
\textbf{AdamPM-inf} &   35.14 G\\
\bottomrule
\end{tabular}
\end{table}

\paragraph{Memory evaluation on LLaMA 7B.}
Due to limited computational resources, we do not perform a complete pre-training run on the LLaMA 7B model. Instead, we evaluate the static memory footprint of AdamPM at the 7B scale, including the model parameters and optimizer states under BF16 precision. Adam requires approximately 39.12 G, whereas AdamPM-2 and AdamPM-inf require only approximately 32.60 G and 26.08 G, respectively.

\subsection{Details of Fine-Tuning on GLUE}
\label{app:hyperparameters of fine-tuning}

For each GLUE task, we select the best-performing learning rate for Adam when fine-tuning RoBERTa-base by searching over $\{1\times10^{-5}, 2\times10^{-5}, 3.5\times10^{-5}, 5\times10^{-5}, 7.5\times10^{-5}, 1\times10^{-4}\}$. We use the selected task-specific learning rate for AdamPM and Adam-mini without further tuning and report the performance of all three optimizers. The fine-tuning hyperparameters are summarized in Table~\ref{tab:roberta_glue_hparams}.

\begin{table}[h]
\caption{Hyperparameters for fine-tuning RoBERTa-base on GLUE.}
\label{tab:roberta_glue_hparams}
\begin{center}
\renewcommand{\arraystretch}{1.5}
\begin{tabular*}{\linewidth}{c c c c c c c c c}
\hline
\textbf{Hyperparameters} & \textbf{MNLI} & \textbf{SST-2} & \textbf{MRPC} & \textbf{CoLA} &
   \textbf{QNLI} & \textbf{QQP} & \textbf{RTE} & \textbf{STS-B} \\
\hline
Batch size & 64 & 32 & 16 & 32 & 64 & 64 & 16 & 32 \\
Epochs & 30 & 30 & 30 & 30 & 30 & 30 & 30 & 30 \\
$lr$  & 2e-5 & 1e-5 & 2e-5 & 2e-5 & 1e-5 & 2e-5 & 2e-5 & 2e-5 \\
$\beta_1$ & \multicolumn{8}{c}{$0.9$} \\
$\beta_2$ & \multicolumn{8}{c}{$0.95$} \\
Max seq. len. & \multicolumn{8}{c}{$512$} \\
\hline
\end{tabular*}
\end{center}
\end{table}

\subsection{Details of TinyLLaVA-1B Fine-Tuning}
\label{app:tinyllava_setup}

Following the setup described in the main text, all methods are initialized from the same pretrained checkpoint and differ only in the SFT optimizer. We perform SFT for one epoch with a global batch size of 128, a learning rate of $2\times10^{-5}$, $\beta_1=0.9$, $\beta_2=0.95$, and a maximum sequence length of 2048. Both the training seed and data-order seed are fixed to 42.








\end{document}